\documentclass[runningheads]{llncs}
\usepackage[T1]{fontenc}
\usepackage{graphicx}
\usepackage{amsmath}
\usepackage{graphicx}
\usepackage{multirow}
\usepackage{amssymb}
\usepackage{tikz-cd}

\newtheorem{Language}{Language}
\newtheorem{principle}[Language]{Principle}

\newtheorem{notation}{Notation}[section]

\newtheorem{formalization}{Formalization}

\DeclareMathOperator{\CM}{CM}

\DeclareMathOperator{\causal}{Causal}
\DeclareMathOperator{\comp}{comp}

\DeclareMathOperator{\head}{head}

\DeclareMathOperator{\body}{body}

\DeclareMathOperator{\graph}{graph}

\DeclareMathOperator{\Pa}{Pa}

\DeclareMathOperator{\Error}{Error}

\begin{document}
\title{How Rules Represent Causal Knowledge: \\ Causal Modeling with Abductive Logic Programs}
\titlerunning{Causal Modeling with Abductive Logic Programs}
\author{Kilian Rückschloß\inst{1}\orcidID{0000-0002-7891-6030} \and \\
Felix Weitkämper\inst{2,3}\orcidID{0000-0002-3895-8279} 
}
\authorrunning{K.~Rückschloß and F.~Weitkämper}
\institute{Eberhard-Karls-Universität Tübingen, Auf der Morgenstelle 10,
72076 Tübingen  \and
German University of Digital Science, Marlene-Dietrich-Allee 14, 14482 Potsdam \and
Ludwig-Maximilians-Universit\"at M\"unchen, Oettingenstr.~67, 80538 M\"unchen}
\maketitle              % typeset the header of the contribution
\begin{abstract}
Pearl observes that causal knowledge enables predicting the effects of interventions, such as actions, whereas descriptive knowledge only permits drawing conclusions from observation.
This paper extends Pearl’s approach to causality and interventions to the setting of stratified abductive logic programs. It shows how stable models of such programs can be given a causal interpretation by building on philosophical foundations and recent work by Bochman and Eelink et al.
In particular, it provides a translation of abductive logic programs into causal systems, thereby clarifying the informal causal reading of logic program rules and supporting principled reasoning about external actions. The main result establishes that the stable model semantics for stratified programs conforms to key philosophical principles of causation, such as causal sufficiency, natural necessity, and irrelevance of unobserved effects.
This justifies the use of stratified abductive logic programs as a framework for causal modeling and for predicting the effects of interventions.
\end{abstract}
\section{Introduction}
After being a central topic of philosophical inquiry for over two millennia, causality entered the mainstream of artificial intelligence research through the work of Pearl~\cite{Causality}.  
A key feature of his account is that causal knowledge goes beyond descriptive knowledge in the questions it can address: while descriptive knowledge permits only inferences from observations, causal knowledge enables reasoning about the effects of external interventions such as actions on the modeled system.

\begin{example}
    Consider a road that passes through a field with a sprinkler. Assume the sprinkler is turned on by a weather sensor when it is sunny. Suppose further that it rains whenever it is cloudy, and that the road becomes wet if either it rains or the sprinkler is activated. Finally, assume that a wet road is dangerous. 
    
    Observing that the sprinkler is on, one might conclude that the weather is sunny. However, actively intervening and switching the sprinkler on does not cause the weather to become sunny. To predict the effect of such an intervention, one needs causal---not merely descriptive---knowledge.
    \label{example - introduction running example}
\end{example}

Since evaluating the effects of possible actions is one of the primary motivations for modeling in the first place, this has paved the way for the adoption of causal frameworks across a wide variety of domains~\cite{ArifM22,GaoZWFHL24,WuPLZSLQLG24}.

Pearl~\cite{Causality}, however, develops his theory of causality exclusively within his own formalisms: Bayesian networks and structural causal models.

\begin{example}
    Recall Example~\ref{example - introduction running example}, and denote by $c$ the event that the weather is cloudy, by $s$ the event that the sprinkler is on, by $r$ the event of rain, by $w$ the event that the road is wet, and by $d$ the event that the road is dangerous.

    Pearl~\cite{Causality} models the causal mechanisms as a system of structural equations:
    \begin{align}
        & r := c && s := \neg c && w := r \lor s && d := w
        \label{structural equations -  sprinkler}
    \end{align}
    Since the mechanisms do not specify whether it is cloudy, Pearl~\cite{Causality} treats $c$ as an external variable or error term.  
    The solutions of the corresponding causal model~$\mathcal{M}$ are obtained by solving Equations~\eqref{structural equations -  sprinkler} for $c = \top$ and $c = \bot$.  
    If the sprinkler is observed to be on, then according to Equations~\eqref{structural equations -  sprinkler}, it must be sunny.

    The intervention of manually switching the sprinkler on is represented by modifying the causal mechanism for the sprinkler so that it is on regardless of the weather. This is captured in the following structural equations:
    \begin{align}
        & r := c && s := \top && w := r \lor s && d := w
        \label{structural equations after intervention -  sprinkler}
    \end{align}
    Again, $c$ is considered an external variable, and the solutions of the corresponding causal model~$\mathcal{M}_{s}$ represent the possible states of the world after the intervention of switching the sprinkler on manually.  
    Note that $s$ is now true independently of the weather.
    \label{example - sprinkler causal model}
\end{example}

In philosophy, the idea that causal explanations are given by rules of the form ``$\phi$ \emph{causes} $\psi$'' is well established, for instance, in René Descartes’ \emph{Principles of Philosophy} II:37 (see the translation by Miller and Miller~\cite{Miller2009}).  
This makes abductive logic programming~\cite{KakasM90} a natural target formalism for representing causal knowledge.  
An abductive logic program consists of a set of rules~$\textbf{P}$ and a set of propositions~$\mathfrak{A}$, called \emph{abducibles}.  
Similar to external variables in Pearl’s causal models, abducibles are independently assumed to be either true or false.  
Rückschloß and Weitkämper~\cite{ruckschloss2021exploiting} apply Clark’s completion~\cite{ClarkCompletion} to relate probabilistic logic programming to Pearl’s theory of causality, thereby transferring Pearl’s notion of intervention.  
Similarly, abductive logic programs can be translated into causal models, enabling a principled treatment of interventions.

\begin{example}
    The situation in Example~\ref{example - introduction running example} gives rise to the rules:
    \begin{align}
        & r \leftarrow c && s \leftarrow \neg c && w \leftarrow r && w \leftarrow s && d \leftarrow w
        \label{rules - sprinkler}
    \end{align}
    Since $c$ can be either true or false, it is considered an abducible; that is, $\mathfrak{A} := \{ c \}$.  
    The Clark completion~\cite{Clark} yields that the models of the resulting abductive logic program are exactly the solutions of the causal model~$\mathcal{M}$ in Example~\ref{example - sprinkler causal model}.

    Intervening and switching the sprinkler on manually results in the rules:
    \begin{align}
        & r \leftarrow c && s \leftarrow \top && w \leftarrow r && w \leftarrow s && d \leftarrow w
        \label{rules - sprinkler intervention}
    \end{align}
    Again, Clark completion~\cite{Clark} yields that the models of the resulting abductive logic program coincide with the solutions of the causal model~$\mathcal{M}_s$ in Example~\ref{example - sprinkler causal model}. 
    \label{example - sprinkler abductive logic program}
\end{example}

Notably, this approach provides an informal semantics for abductive logic programs, where each rule $h \leftarrow b_1, \ldots, b_n$ is interpreted as ``$b_1 \land \ldots \land b_n$ \emph{causes}~$h$''.  
However, translating abductive logic programs into causal models can lead to counterintuitive results when cyclic causal relations are involved.

\begin{example}
    Assume $h_1$ and $h_2$ are two neighboring houses. Let $f_i$ denote the event that House~$h_i$ is on fire, and $sf_i$ the event that House~$h_i$ starts burning, for~${i = 1,2}$.  
    It is reasonable to assume that a fire in House~$h_1$ leads to a fire in House~$h_2$, and vice versa.  

    This situation can be modeled by the cyclic abductive logic program $\mathcal{P}$, consisting of the abducibles $\mathfrak{A} := \{ sf_1, sf_2 \}$ and the rules:
    \begin{align}
        & f_1 \leftarrow sf_1 && f_2 \leftarrow sf_2 && f_2 \leftarrow f_1 && f_1 \leftarrow f_2
        \label{rules - houses}
    \end{align}

    Proceeding as in Example~\ref{example - sprinkler abductive logic program} and relating $\mathcal{P}$ to a causal model, the program admits the model $\omega := \{ f_1, f_2 \}$.  
    In $\omega$, both houses are on fire, even though neither house actually started to burn.  
    This contradicts the intuition that houses do not spontaneously catch fire merely because they affect each other.
    \label{example - houses abductive logic program}
\end{example}

This work extends the applicability of Pearl’s ideas to stratified abductive logic programs that involve cyclic causal relations. Building on prior work by Bochman~\cite{bochman} and Eelink et al.~\cite{EelinkRW25}, it connects logic programming and Pearl’s causal models to the philosophical Principles~\ref{principle - causal foundation}--\ref{principle - non-interference} stated below. Theorem~\ref{theorem - bochman transformation} translates abductive logic programs into the causal systems of Eelink et al.~\cite{EelinkRW25}, thereby relating the stable model semantics to Principles~\ref{principle - causal foundation}, \ref{principle - Aquinas}, and~\ref{principle - Leibniz}. Theorem~\ref{theorem - non-interference} shows that Principle~\ref{principle - causal irrelevance} implies~\ref{principle - non-interference}. Finally, Theorem~\ref{theorem - stratified programs satisfy causal irrelevance} establishes that stratified programs satisfy Principle~\ref{principle - causal irrelevance}. 

Overall, the results show that stratified abductive logic programs under the stable model semantics conform to these principles, supporting their use in causal modeling and the prediction of effects from external interventions.

\begin{principle}[Causal Foundation]
     \textit{Causal explanations} originate from external premises, whose explanations lie beyond the scope of the given model. 
     \label{principle - causal foundation}
\end{principle}

\begin{principle}[Natural Necessity]
    ``\ldots given the existence of the cause, the effect must necessarily follow.'' (Thomas Aquinas: \emph{Summa Contra Gentiles} II:35.4; translation by Anderson~\cite{Anderson})
    \label{principle - Aquinas}
\end{principle}  

\begin{principle}[Sufficient Causation]
    ``\ldots there is nothing without a reason, or no effect without a cause.''  
    (Gottfried Wilhelm Leibniz: \emph{First Truths}; translation by Loemker~\cite{Leibniz}, p.~268)
    \label{principle - Leibniz}
\end{principle}

 \begin{principle}[Causal Irrelevance \cite{Williamson2001}]
        Unobserved effects have no influence on beliefs.
         \label{principle - causal irrelevance}
 \end{principle}

 \begin{principle}[Non-Interference]
     The impact of interventions is restricted to the causal direction from causes to effects.    
     \label{principle - non-interference}
 \end{principle}

\section{Preliminaries}

Here, we present the material on which this contribution builds, namely Pearl's causal models~\cite{Causality}, abductive logic programs~\cite{KakasM90}, and the logical theories of causality developed by Bochman~\cite{bochman} and Eelink et al.~\cite{EelinkRW25}.

\subsection{Pearl's Causal Models}
\label{subsec: Pearl's Causal Models}
Pearl \cite{Causality} suggests modeling causal relationships with deterministic functions. This leads to the following definition of structural causal models. 

\begin{definition}[Causal Model {\cite[§7.1.1]{Causality}}]
A \textbf{(structural) causal model}~$\mathcal{M}$ with \textbf{internal variables}~$\textbf{V}$ and \textbf{external variables}~$\textbf{U}$ is a system of equations that includes one \textbf{structural equation} of the form
$
X := f_X(\Pa(X), \Error(X))
$
for each internal variable~$X \in \textbf{V}$. Here, the \textbf{parents}~$\Pa(X) \subseteq \textbf{V}$ of~$X$ are a subset of internal variables, the \textbf{error term}~$\Error(X) \subseteq \textbf{U}$ is a subset of external variables, and~$f_X$ is a function that maps each assignment of values to~$\Pa(X)$ and~$\Error(X)$ to a value of~$X$.

A \textbf{solution}~$\omega$ of the structural causal model~$\mathcal{M}$ is an assignment of values to the variables in~$\textbf{V} \cup \textbf{U}$ that satisfies all structural equations.
\label{definition - structural causal model}
\end{definition}

\begin{notation}
    The parents $\Pa(V)$ and error terms $\Error(V)$ of an internal variable ${V \in \textbf{V}}$ are typically evident from its defining function~$f_V$.  
    Accordingly, we omit explicit references to $\Pa(\cdot)$ and $\Error(\cdot)$ throughout this work.
\end{notation}

\begin{example}
The causal model $\mathcal{M}$ in Example \ref{example - sprinkler causal model} has  external variables $\textbf{U} := \{ c \}$, internal variables~${\textbf{V} := \{ r,s,w,d \}}$ and Structural Equations (\ref{structural equations -  sprinkler}). 
and solutions: 
 \begin{align*}
     & \omega_1: && c = \top && r = \top && s = \bot && w = \top && d = \top \\
     & \omega_2: && c = \bot && r = \bot && s = \top && w = \top && d = \top
\end{align*}
\label{example - sprinkler deterministic}
\end{example}

In artificial intelligence, causal models are particularly valuable because they can represent external interventions.  
As explained in Chapter 7 of Pearl~\cite{Causality}, the key idea is to construct a modified model that incorporates the minimal changes to the structural equations required to enforce an external intervention.

\begin{definition}[Modified Causal Model]
    Fix a causal model $\mathcal{M}$.  
    Let~${\mathbf{I}}$ be a subset of internal variables with a value  assignment $\mathbf{i}$.  
    The \textbf{modified model} or \textbf{submodel} $\mathcal{M}_{\mathbf{i}}$ is the model obtained from $\mathcal{M}$ by replacing, for each variable~${X \in \mathbf{I}}$, the structural equation
    $
        X := f_X(\Pa(X), \Error(X))
    $
    with~$
        {X := \mathbf{i}(X)}.
    $
    \label{definition - modified causal model}
\end{definition}

\begin{notation}
    Let $V \in \textbf{V}$ be a Boolean internal variable of a structural causal model~${\mathcal{M}}$. In this case, we write $\mathcal{M}_V := \mathcal{M}_{V := \top}$ and $\mathcal{M}_{\neg V} := \mathcal{M}_{V := \bot}$.
\end{notation}

\begin{example}
    The causal model $\mathcal{M}_s$ from Example~\ref{example - sprinkler causal model} is the modified model corresponding to the value assignment $s := \top$.  
    It has the following solutions:
    \begin{align*}
        & \omega_1: && c = \top, && r = \top, && s = \top, && w = \top, && g = \top \\
        & \omega_2: && c = \bot, && r = \bot, && s = \top, && w = \top, && g = \top
    \end{align*}
    These represent the possible states of the system after manually switching on the sprinkler.
    \label{example - sprinkler deterministic intervention}
\end{example}

As in Example~\ref{example - sprinkler causal model}, actions often force a variable in a causal model to take on a new value. Pearl~\cite{Causality} emphasizes that submodels~$\mathcal{M}_{\textbf{i}}$ typically arise from performing actions that set certain variables to specific values, a process formalized by the introduction of the \emph{do}-operator.

\subsection{Abductive Logic Programming}
\label{subsec: Abductive Logic Programming}
This work adopts standard notation for propositions, (propositional) formulas, and structures. A structure is identified with the set of propositions that are true in it. The term \emph{world} refers to a consistent set of literals that is maximal with respect to inclusion. Since worlds~$\omega$ are in one-to-one correspondence with structures, the two terms are used interchangeably.

\begin{example}
    \label{example - propositional alphabet}
        A structure $\omega$ in the propositional alphabet  $      
    \mathfrak{P} := \{ c,~r,~s,~w,~d \}
        $ of Example \ref{example - sprinkler deterministic} is a complete state description such as $\omega_1$ in Example \ref{example - sprinkler deterministic}. 
    We identify the structure $\omega_1$ with the set of propositions $\{ c,~r,~w,~d\}$.
    \label{example - propositional structure}

\end{example}

Fix a propositional alphabet $\mathfrak{P}$. Logic programs consist of rules or clause.  

\begin{definition}[Clauses and Logic Programs]
    A \textbf{(normal) clause} $C$ is a formula of the form $
    (h \leftarrow (b_1 \land (b_2 \land ( ... \land b_n))...)),
    $ which we also denote as~${
    h \leftarrow b_1 \land ... \land b_n}$,
    $ h \leftarrow b_1,..., b_n
    \text{ or }
    \head (C) \leftarrow \body (C)
    $.
    Here, ${\head(C) := h}$ is an atom, referred to as the \textbf{head} of the clause~$C$ and~${\body(C) := \{ b_1,...,b_n \}}$ is a finite set of literals, known as the \textbf{body} of~$C$. If~${\body (C) = \emptyset}$ and  $C = (h \leftarrow \top)$, one denotes $C$ by~$h$ and calls $C$ a \textbf{fact}. 
    %Furthermore, we say that $C$ is \textbf{positive} if $\body(C) \subseteq \mathfrak{P}$, i.e., it does not contain any negative literals. 
%    \label{definition - propositonal clause}
% \end{definition}

% Programs are constructed from these clauses.

% \begin{definition}[Logic Program, Dependence Graph]
   A \textbf{logic program}~$\textbf{P}$ is a finite set of clauses.  
   The \textbf{dependence graph} of~$\textbf{P}$ is the directed graph over the alphabet~$\mathfrak{P}$ defined as follows:  
   there is an edge $p \rightarrow q$ if and only if there exists a clause $C \in \textbf{P}$ such that $\head(C) = q$ and $\body(C) \cap \{ p, \neg p \} \neq \emptyset$. It is denoted by~$\graph(\textbf{P})$.
   
   An edge $p \stackrel{-}{\rightarrow} q$ in $\graph(\textbf{P})$ is \textbf{negative} if there exists a clause $C \in \textbf{P}$ such that $\head(C) = q$ and $\neg p \in \body(C)$.  
   Similarly, an edge $p \stackrel{+}{\rightarrow} q$ is \textbf{positive} if there exists a clause $C \in \textbf{P}$ such that $\head(C) = q$ and $p \in \body(C)$.  
   Note that an edge may be both negative and positive simultaneously.

   A \textbf{cycle} in $\graph(\textbf{P})$ is a finite alternating sequence of nodes and edges of the form 
   ${
       q \rightarrow p_1 \rightarrow p_2 \rightarrow \dots \rightarrow p_n \rightarrow q
   }$
   that begins and ends at the same node~$q$.  
   The program~$\textbf{P}$ is \textbf{acyclic} if its dependence graph~$\graph(\textbf{P})$ contains no cycle.  
   It is \textbf{stratified} if its dependence graph does not contain a cycle with a negative~edge.
\end{definition}

Clark \cite{ClarkCompletion} translates acyclic programs $\textbf{P}$ to propositional formulas, stating that a valid proposition in a model~$\omega$ needs to have a reason, i.e., a support in~$\omega$.

\begin{definition}[Clark Completion, Supported Model Semantics]
Let $\textbf{P}$ be a logic program. The \textbf{Clark completion} of $\textbf{P}$ is the set of formulas
$$
\comp (\textbf{P}) := 
\left\{ 
p \leftrightarrow \bigvee_{\substack{C \in \textbf{P},~\head (C) = p}} \bigwedge \body (C) 
\right\}_{ p \in \mathfrak{P}}.
$$
A \textbf{supported model} of $\textbf{P}$ is model of the Clark completion $\omega \models \comp (\textbf{P})$. 
%\label{definition - clark completion}
\end{definition}

\begin{example}
    Rules (\ref{rules - sprinkler}) define a logic program $\textbf{P}$ whose unique supported model is $\omega_2$, as given in Example~\ref{example - sprinkler deterministic}. Rules (\ref{rules - houses}) yield a stratified program.
    \label{example - logic program sprinkler}
\end{example}

Although the supported model semantics is formally well-defined for general propositional logic programs, i.e., it associates a unique (possibly empty) set of models to each program $\textbf{P}$, it yields counterintuitive results for cyclic programs.

\begin{example}
    Rules (\ref{rules - houses}) define a logic program $\textbf{P}$ with two supported models: ${\omega_1 := \emptyset}$ and $\omega_2 := \{ f_1, f_2 \}$.  
    In $\omega_2$, both houses are on fire, even though there is no initial cause for either to start burning -- contradicting everyday intuition.
    \label{example - supported models of houses}
\end{example}

For general, potentially cyclic programs, Gelfond and Lifschitz~\cite{StableModelSemantics} argue that the stable model semantics provides a more appropriate notion of a model.  
Rather than adopting the more common formulation via reducts, this work follows the equivalent definition based on \emph{unfounded sets}, originally introduced by Saccà and Zaniolo~\cite{SaccaZ90} and listed as Definition~D in Lifschitz~\cite{Lifschitz10}.
 %This formulation aligns more closely with the intuition of well-founded explanations, which also underlies the semantics for causal systems proposed by~\mbox{\cite{EelinkRW25}} and introduced in Definition~\ref{definition - causal worlds}.
 
\begin{definition}[Unfounded Sets and Stable Models]
Let $\omega$ be a structure,~${I \subseteq \omega}$ a non-empty subset of the set of atoms that are true in $\omega$, and~$\textbf{P}$ a logic program.  
Then, $I$ is an \textbf{unfounded set} with respect to $\omega$ and $\textbf{P}$ if, for each $p \in I$, every rule in $\textbf{P}$ with head $p$ has some body literal $b$ that is either not true in $\omega$ or belongs to $I$.

A structure $\omega$ is a \textbf{stable model} of $\textbf{P}$ if it satisfies every clause of $\textbf{P}$ when interpreted as a propositional formula and if there is no unfounded set $I \subseteq \omega$ with respect to $\omega$ and $\textbf{P}$.
\end{definition}

\begin{example}
    In Example~\ref{example - supported models of houses}, the only stable model is~$\omega_1$, as intended.  
    The set~$\omega_2$ is not stable because it is unfounded with respect to itself and the program~$\textbf{P}$.
\end{example}

Gelfond and Lifschitz \cite{StableModelSemantics} prove the following results:

\begin{theorem}[Supported and Stable Models]
    Every stable model of a logic program is also a supported model~$\square$.
    \label{theorem - consistency of Clark completion}
\end{theorem}

\begin{theorem}[Stable Models of Stratified Programs]
    Every stratified program has a unique stable model.~$\square$
    \label{theorem - stable model semantics of stratified programs}
\end{theorem}

% However, supported models are not necessarily stable.

% \begin{example}
%     Consider the program \textbf{P} that is given by
%     \begin{align*}
%         & p \leftarrow \neg p && p \leftarrow q && q \leftarrow q.
%     \end{align*}
%     Note that $\textbf{P}$ has the unique supported model $\{p,q\}$, which is not stable since~${\{p,q\}}$ is an unfounded set. 
%     % Observe that this does not contradict Corollary \ref{corollary - unique supported and stable model} as the program $\textbf{P}$ is not consistent.
% \end{example}

Abductive logic programming was identified as a distinct branch of logic programming by Kakas and Mancarella \cite{KakasM90}, with the goal of providing an explanation for a given set of observations in terms of so-called abducibles. 

\begin{definition}[Abductive Logic Program \cite{AbductiveLogicPrograms}]
    An \textbf{integrity constraint}~$\textit{IC}$ is an expression of the form $\bot \leftarrow b_1 \land ... \land b_n$ also written $\bot \leftarrow \body (IC)$, where~$\body (\textit{IC})$ is a finite set of literals.  
    
    An \textbf{abductive logic program} is a triplet $\mathcal{P} := (\textbf{P},\mathfrak{A},\textbf{IC})$ consisting of a logic program~$\textbf{P}$, a finite set of integrity constraints $\textbf{IC}$ and a set of \textbf{abducibles}~$\mathfrak{A} \subseteq \mathfrak{P}$ such that no abducible~${u \in \mathfrak{A}}$ is the head of a clause in $\textbf{P}$. Finally, $\mathcal{P}$ is \textbf{acyclic} or \textbf{stratified} if the underlying logic program $\textbf{P}$ is. 
    \label{definition - abductive logic program}
\end{definition}

In the context of databases, integrity constraints serve as sanity checks on data~\mbox{\cite[Chapter~9]{databases}}.  
In our setting, they are used to represent observations; that is, they ensure that the knowledge encoded by the causal rules in the program~$\textbf{P}$ and the explanations in~$\mathfrak{A}$ is consistent with the given observations.

\begin{example}
    Let $\textbf{P}$ denote the logic program in Example \ref{example - logic program sprinkler}. Since we expect that the causal knowledge in Example \ref{example - introduction running example} is insufficient to explain whether it is cloudy, we declare~$c$ as the only abducible and define~$\mathfrak{A} := \{ c \}$. We may also observe that the sprinkler is on, leading to the integrity constraint~$\textbf{IC} := \{ \bot \leftarrow \neg s \}$. Together, this yields the abductive logic program~$\mathcal{P} := (\textbf{P}, \mathfrak{A}, \textbf{IC})$.
    \label{example - abductive logic program}
\end{example}

Lastly, we recall the various semantics of an abductive logic program.

\begin{definition}[Models of Abductive Logic Programs]
    A \textbf{stable} or \textbf{supported model}~${\omega \subseteq \mathfrak{P}}$ of the abductive logic program~${\mathcal{P} := (\textbf{P},\mathfrak{A},\textbf{IC})}$ satisfies the integrity constraints $\textbf{IC}$, i.e., $\omega \models \textit{IC}$ (meaning~${\omega \not \models \body (\textit{IC})}$) for all $\textit{IC} \in \textbf{IC}$ and is a stable or supported model of the program~${\textbf{P} \cup (\omega \cap \mathfrak{A})}$.  The set $\epsilon := \omega \cap \mathfrak{A}$ is then called the \textbf{explanation} of~$\omega$.
    In this context, the program $\mathcal{P}$ is \textbf{consistent} if it has at least one model for every choice of abducibles.
    \label{definition - semantics of an abductive logic program}
\end{definition}

\begin{example}
    The abductive logic program in Example~\ref{example - sprinkler abductive logic program} has two stable and supported models, namely~$\omega_1$ and~$\omega_2$ from Example~\ref{example - sprinkler deterministic}, 
    with explanations~${\epsilon_1 := \{ c \}}$ and~$\epsilon_2 := \emptyset$, respectively. 
    
    Since only~$\omega_2$ is consistent with the observation that the sprinkler is on, expressed by the integrity constraint~$\bot \leftarrow \neg s$, it is the only supported model of the abductive logic program~$\mathcal{P}$ in Example~\ref{example - abductive logic program}.
    \label{example - semantics of abductive logic programs}
\end{example}

Following the lines of Rückschloß and Weitkämper \cite{ruckschloss2021exploiting}, who connect probabilistic logic programming to Pearl's theory of causality, we realize that abductive logic programs without integrity constraints under the supported model semantics give rise to causal models thereby transferring the notion of an intervention: 

\begin{definition}[CM-Semantics]
    Let ${\mathcal{P} := (\textbf{P},\mathfrak{A}, \emptyset)}$ be an abductive logic program without integrity constraints. The \textbf{causal model semantics} of $\mathcal{P}$ is the causal model $\CM(\textbf{P})$ that is given by the external variables $\mathfrak{A}$, the internal variables $\mathfrak{P} \setminus \mathfrak{A}$ and the structural equations $\displaystyle p := \bigvee_{\substack{C \in \textbf{P} \\ \head (C) = p}} \bigwedge \body (C)$ for $p \in \mathfrak{A}$. Note that by construction a structure $\omega$ is solution of $\CM (\textbf{P})$ if and only if it is a supported model of $\textbf{P}$.

    To represent the intervention of forcing the atoms in $\textbf{I} \subseteq \mathfrak{P} \setminus \mathfrak{A}$ to attain values according to the assignment $\textbf{i}$, the \textbf{modified (abductive logic) program}~${
    \mathcal{P}_\textbf{i} := (\textbf{P}_{\textbf{i}}, \mathfrak{A}_{\textbf{i}}, \emptyset)
    }$
    is obtained from $\mathcal{P}$ by the modifications below:
    \begin{itemize}
        \item[] Remove all clauses $C$ from $\textbf{P}$ for which $\textrm{head}(C) \in \textbf{i}$ or $\neg \textrm{head}(C) \in \textbf{i}$.
        %\item[] Remove any abducible $p$ from $\mathfrak{A}$ for which $\neg p \in \textbf{i}$.
        %\item Remove external premises $\neg p \in \mathcal{E}$ if ${p \in \textbf{i}}$.
        \item[] Add a fact $p$ to $\textbf{P}_{\textbf{i}}$ whenever ${p \in \textbf{i}}$.
    \end{itemize} 
    Note that by construction $\CM(\textbf{P}_{\textbf{i}}) = \CM (\textbf{P})_{\textbf{i}}$.
    \label{definition - CM semantics}
\end{definition}

\begin{example}
    In Example \ref{example - sprinkler abductive logic program}, Rules (\ref{rules - sprinkler intervention}) correspond to the modified program $\mathcal{P}_s$ that corresponds to the assignement $s := \top$.
\end{example}

\subsection{Bochman's Logical Theory of Causality and Causal Systems}
\label{subsec: Bochman's Logical Theory of Causality}
% Introduce causal systems and their underlying principles. 
% Directly referenced: 
% - principle - default negation
% - example - propositional alphabet

To verify that the stable model semantics is causally meaningful, this work builds upon the work of Bochman~\cite{bochman} and Eelink et al.~\cite{EelinkRW25}. Both rely on the idea that causal knowledge should be expressed in the form of rules.

\begin{definition}[Causal Rules and Causal Theories]
    A \textbf{(literal) causal rule} $R$ is an expression of the form
    $
        b_1 \land ... \land b_n \Rightarrow l,
    $
    also denoted by~${
        \{ b_1, ... , b_n \} \Rightarrow l,
    }$
    where $b_1, ... , b_n, l$ are literals. We call $b_1 \land \cdots \land b_n$ the \textbf{cause} and $l$ the \textbf{effect} of~$R$. Informally, $R$ means that $b_1 \land \cdots \land b_n$ causes $l$.  
    If, in addition, $l \in \mathfrak{P}$ is an atom, the rule $R$ is \textbf{atomic}.  
    A \textbf{default rule} is a causal rule of the form $l \Rightarrow l$.  
    A \textbf{causal theory} is a set of causal rules $\Delta$.  
    It is called \textbf{atomic} if it contains only atomic causal rules.
    \label{definition - causal theory}
\end{definition}

\begin{example}
    Example~\ref{example - houses abductive logic program} gives rise to the following causal theory~$\Delta$:
    \begin{align*}
     & sf_1 \Rightarrow sf_1 && sf_2 \Rightarrow sf_2 && sf_1 \Rightarrow f_1 && sf_2 \Rightarrow f_2 && f_1 \Rightarrow f_2 && f_2 \Rightarrow f_1 \\
     & \neg sf_1 \Rightarrow \neg sf_1 && \neg sf_2 \Rightarrow \neg sf_2 && \neg f_1 \Rightarrow \neg f_1 && \neg f_2 \Rightarrow \neg f_2
    \end{align*}

    The default rule~$\neg f_1 \Rightarrow \neg f_1$ expresses that no explanation is required for House~1 not burning; that is, House~1 is assumed not to burn unless an explanation for~$f_1$ is given.  
    Since both default rules~$sf_1 \Rightarrow sf_1$ and~$\neg sf_1 \Rightarrow \neg sf_1$ are included in~$\Delta$, the truth value of~$sf_1$ can be chosen freely.

    % Example~\ref{example - introduction running example} gives rise to the following causal theory~$\Delta_1$:
    % \begin{align*}
    %     & c \Rightarrow c && c \Rightarrow r && \neg c \Rightarrow s && r \Rightarrow w && s \Rightarrow w && w \Rightarrow g \\
    %     & \neg c \Rightarrow \neg c && \neg r \Rightarrow \neg r && \neg s \Rightarrow \neg s && \neg w \Rightarrow \neg w && \neg g \Rightarrow \neg g
    %     \label{rules - sprinkler}
    % \end{align*}
    % The default rule~$\neg s \Rightarrow \neg s$ means that no explanation is needed for the sprinkler to be off; that is, the sprinkler is assumed to be off unless an explanation for~$s$ is provided.  
    % Since both default rules~$c \Rightarrow c$ and~$\neg c \Rightarrow \neg c$ are included in~$\Delta$, the truth value of~$c$ can be chosen freely.

    % Analogously, Example \ref{example - houses abductive logic program} gives rise to the following causal theory~$\Delta_2$:
    % \begin{align*}
    %     & sf_1 \Rightarrow sf_1 && sf_2 \Rightarrow sf_2 && sf_1\Rightarrow f_1 && sf_2 \Rightarrow f_2 && f_1 \Rightarrow f_2 && f_2 \Rightarrow f_1 \\
    %     & \neg sf_1 \Rightarrow \neg sf_1 && \neg sf_2 \Rightarrow \neg sf_2 && \neg f_1 \Rightarrow \neg f_1 && \neg f_2 \Rightarrow \neg f_2
    % \end{align*}
    \label{example - causal theories}
\end{example}

\mbox{Bochman \cite{bochman} extends the rules in a causal theory to a \textit{explainability relation}.}   

\begin{definition}
     Let $\Delta$ be a causal theory.
     The binary relation $(\Rrightarrow_\Delta)/2$  of \textbf{explainability} is defined inductively from the causal rules as follows:
    \begin{enumerate}
         \item[] If  $\lambda \Rightarrow l$, then $\lambda \Rrightarrow_\Delta l$. 
         \textbf{(Causal rules)}
         \item[] If $\lambda \Rrightarrow_\Delta l$, then $\lambda \cup \lambda' \Rrightarrow_\Delta l$. \textbf{(Literal Monotonicity)}
         \item[] If $\lambda' \Rrightarrow_\Delta l$ and $\lambda \cup \{ l \} \Rrightarrow_\Delta l'$, then $\lambda \cup \lambda' \Rrightarrow_\Delta l'$. \textbf{(Literal Cut)}
         \item[] $\{ p, \neg p \} \Rrightarrow_\Delta l$ for all propositions $p$ and literals $l$. \textbf{(Literal Contradiction)}
     \end{enumerate}
     If we find $\lambda \Rrightarrow_\Delta l$, we say that $\lambda$ \textbf{explains} $l$.
     \label{definition - determinate causal theories}
\label{definition - production inference relation}
\label{theorem - causal reasoning of literal causal theories}
\end{definition}

\begin{remark}
    Bochman~\cite{bochman} initially allows causal rules of the form~$\phi \Rightarrow \psi$ and explainability relations of the form~$\phi \Rrightarrow \psi$, where~$\phi$ and~$\psi$ are arbitrary formulas.  
    He argues that explainability~$(\Rrightarrow)/2$ satisfies all the properties of material implication ``$\rightarrow$'', except reflexivity (i.e.,~$\phi \rightarrow \phi$ does not necessarily hold).  
    Given a causal theory~$\Delta$ in the sense of Definition~\ref{definition - causal theory}, this work restricts attention to explainability relations as characterized in Definition~\ref{theorem - causal reasoning of literal causal theories}.  
    Theorem~4.23 in Bochman~\cite{bochman} then provides the basis for this characterization.
\end{remark}

\begin{example}
    In Example~\ref{example - causal theories}, we find that $f_1 \Rrightarrow_{\Delta} f_1$ and $f_2 \Rrightarrow_{\Delta} f_2$, even though the causal theory~$\Delta$ does not explicitly assert that~$f_1$ or $f_2$ are defaults.  
    \label{example - explanation}
\end{example}

Bochman's semantics \cite{bochman} for causal theories is grounded in Principles~\ref{principle - Aquinas} and~\ref{principle - Leibniz}:

% \begin{principle}[Natural Necessity]
%     ``\ldots given the existence of the cause, the effect must necessarily follow.'' (Thomas Aquinas: \emph{Summa Contra Gentiles} II:35.4; translation by Anderson~\cite{Anderson})
%     \label{principle - Aquinas}
% \end{principle}  

% \begin{principle}[Sufficient Causation]
%     ``\ldots there is nothing without a reason, or no effect without a cause.''  
%     (Gottfried Wilhelm Leibniz: \emph{First Truths}; translation by Loemker~\cite{Leibniz}, p.~268)
%     \label{principle - Leibniz}
% \end{principle}

\begin{definition}[Causal World Semantics]
    A \textbf{causal world} for a causal theory~$\Delta$ is a world~$\omega$ such that, for every literal~$l$, the following formalization of Principles \ref{principle - Aquinas} and \ref{principle - Leibniz} hold:
     \begin{multicols}{2}
    \begin{itemize}
        \item[] \textbf{Formalization of Principle~\ref{principle - Aquinas}}: 
        \item[] 
         If~$\omega \Rrightarrow_{\Delta} l$, then~${l \in \omega}$.
        \item[] \textbf{Formalization of Principle~\ref{principle - Leibniz}}: 
        \item[] 
        If~$l \in \omega$, then~${\omega \Rrightarrow_{\Delta} l}$.
    \end{itemize}
    \end{multicols}
    
    The \textbf{causal world semantics}~$\causal(\Delta)$ is the set of all causal worlds of~$\Delta$.
    \label{definition - causal worlds}
\end{definition}

Bochman \cite{bochman} gives the following alternative characterisation for the causal worlds~$\omega$ of a causal theory $\Delta$.

\begin{definition}[Completion of Causal Theories]
    The \textbf{completion} of a causal theory $\Delta$ is the set of formulas
    $ \displaystyle
    \comp (\Delta) := 
    \{ 
    l \leftrightarrow \bigvee_{\phi \Rightarrow l} \phi 
    \}_{l \text{ literal}}.
    $
    \label{definition - completion of a determinate causal theory}
\end{definition}

\begin{theorem}[Completion of Causal Theories $\text{\cite[Theorem 8.115]{bochman2005}}$]
     The causal world semantics~$\causal (\Delta)$ of a causal theory $\Delta$ coincides with the set of all models of its completion:
     $
     \causal (\Delta) = \{ \omega \textrm{ world: } \omega \models \comp (\Delta) \}
     \text{. } \square
     $
     \label{theorem - completion of a causal theory}
\end{theorem}

\begin{example}
    In Example \ref{example - causal theories},  $\Delta$ has the causal world ${\omega := \{ f_1, f_2 \}}$, which contradicts everyday causal reasoning as explained in Example \ref{example - houses abductive logic program}. Note that $\omega$ is a causal world of $\Delta$ since the framework of causal theories allows for the cyclic explanations in Example \ref{example - explanation}.
    \label{example - causal world semantics of causal theories}
\end{example}

To avoid cyclic explanations as in Example \ref{example - explanation}. Eelink et al.~\cite{EelinkRW25} extend Bochman's causal theories~\cite{bochman} to accommodate a set of external premises $\mathcal{E}$ that do not require further explanation. Motivated by the ideas in Aristotles's \textit{Posterior Analytics}, they additionally apply Principle \ref{principle - causal foundation}. This leads them to the set-up of causal systems:

\begin{definition}[Causal System]
    A \textbf{causal system} $\textbf{CS} := (\Delta, \mathcal{E}, \mathcal{O})$ consists~of a causal theory $\Delta$ called the \textbf{causal knowledge} of $\textbf{CS}$, a set of literals~$\mathcal{E}$ called the \textbf{external premises} of $\textbf{CS}$ and
    a set of formulas $\mathcal{O}$ called the \textbf{observations} of $\textbf{CS}$.
    The causal system $\textbf{CS}$ is \textbf{without observations} if $\mathcal{O} = \emptyset$. Otherwise, the causal system $\textbf{CS}$ \textbf{observes something}. The causal system $\textbf{CS}$ applies \textbf{default negation} if every negative literal $\neg p$ for $p \in \mathfrak{P}$ is an external premise, i.e.,~${\neg p \in \mathcal{E}}$ and no external premise is an effect of a causal rule in $\Delta$.
    Further, the system $\textbf{CS}$ is \textbf{atomic} if $\Delta$ is an atomic causal theory. %Finally, we call the causal system $\textbf{CS}$ \textbf{positive} if the causal theory $\Delta$ is.

    The causal theory $
    \Delta(\mathbf{CS}) := \Delta \cup \{ l \Rightarrow l \mid l \in \mathcal{E} \}
    $
    is called the \textbf{explanatory closure} of~$\mathbf{CS}$. A \textbf{causally founded explanation} is an explanation~${\lambda \Rrightarrow_{\Delta(\mathbf{CS})} l}$ such that $\lambda \subseteq \mathcal{E}$. Causally founded explanations formalize Principle \ref{principle - causal foundation}.  
    \begin{itemize}
        \item[] \textbf{Formalization of Principle~\ref{principle - causal foundation}:} 
        Every explanation is causally founded .
\end{itemize}

     A \textbf{causally founded world}~$\omega$ is a world such that $\omega \models \mathcal{O}$ and for every literal $l$ the following formalizations of Principles~\ref{principle - Aquinas} and \ref{principle - Leibniz} are satisfied:
     \begin{itemize}
        \item[] \textbf{Formalization of Principle~\ref{principle - Aquinas}}: 
        \item[] If~there exists a causally founded explanation $\omega \cap \mathcal{E} \Rrightarrow_{\Delta(\textbf{CS})} l$, then~${l \in \omega}$.
        \item[] \textbf{Formalization of Principle~\ref{principle - Leibniz}}: 
        \item[]If~$l \in \omega$, then there exists a causally founded explanation $\omega \cap \mathcal{E} \Rrightarrow_{\Delta(\textbf{CS})} l$.
    \end{itemize} 
    \label{definition - abductive causal theory}
    \label{definition - causal system}
\end{definition}

\begin{example}
    Example~\ref{example - houses abductive logic program} gives rise to the causal system with default negation and without observations, defined as~${\textbf{CS} := (\Delta, \mathcal{E}, \emptyset)}$, where~${\Delta := \{ f_1 \Rightarrow f_2, f_2 \Rightarrow f_1 \}}$ is an atomic causal theory and ~${\mathcal{E} := \{ sf_i, \neg sf_i, \neg f_i \}_{i = 1,2}}$ the set of external premises. The explanatory closure~$\Delta(\textbf{CS})$ of~$\textbf{CS}$ coincides with the causal theory in Example~\ref{example - causal theories}. 
    Note that the cyclic explanations from Example~\ref{example - explanation} are not causally founded. Hence, the world~$\omega$ in Example~\ref{example - causal world semantics of causal theories} is not causally founded.
    \label{example - causally founded worlds of a causal system}
\end{example}

\section{Problem Statement}

Example~\ref{example - houses abductive logic program} shows that abductive logic programming under the causal model semantics can yield counterintuitive results in the presence of cyclic causal relationships. From the perspective of logic programming, such issues are typically addressed by applying the stable model semantics of Gelfond and Lifschitz~\cite{StableModelSemantics}. However, it remains an open question whether this approach admits a causally meaningful interpretation that accounts for interventions.

\section{Main Result}

Throughout this section, we fix a propositional alphabet~$\mathfrak{P}$. We begin by introducing the Bochman transformation, which identifies abductive logic programs with causal systems featuring default negation~\cite{EelinkRW25}, as defined in Definition~\ref{definition - causal system}.

Informally, the Bochman transformation interprets clauses
$
h \leftarrow b_1 \land \dots \land b_n 
$
as ``$b_1 \land \dots \land b_n$ \textit{causes} $h$,'' treats the abducibles as external premises whose explanations lie beyond the given model, and regards the integrity constraints as observations.

\begin{definition}[Bochman Transformation]
    The \textbf{Bochman transformation} of an abductive logic program 
    $\mathcal{P} := (\textbf{P}, \mathfrak{A}, \textbf{IC})$ is the causal system 
    with default negation $\textbf{CS}(\mathcal{P}) := (\Delta, \mathcal{E}, \mathcal{O})$,
    where $\Delta := \{ \body(C) \Rightarrow \head(C) \mid C \in \textbf{P} \}$, 
    $\mathcal{E} := \mathfrak{A} \cup \{ \neg p \mid p \in \mathfrak{P} \}$, and 
    $\mathcal{O} := \textbf{IC}$.
    \label{definition - Bochman transformation}
\end{definition}

\begin{example}
    Let~$\mathcal{P}$ be the abductive logic program in Example \ref{example - houses abductive logic program}. The causal system in Example \ref{example - causally founded worlds of a causal system} is the Bochman transformation $\textbf{CS}(\mathcal{P})$ of $\mathcal{P}$.
\end{example}

Let $\mathcal{P} := (\textbf{P}, \mathfrak{A}, \textbf{IC})$ be an abductive logic program. We show that the stable models of~$\mathcal{P}$ correspond to the causally founded worlds of its Bochman transformation~$\textbf{CS}(\mathcal{P}) := (\Delta, \mathcal{E}, \mathcal{O})$. Together with the formalizations in Definition~\ref{definition - causal system}, this supports that the stable model semantics follows from Principles~\ref{principle - causal foundation}, \ref{principle - Aquinas}, and~\ref{principle - Leibniz}.

We begin by relating unfounded sets to explainability in causal theories.

\begin{definition}[Internal and External Explanations]
    Let $\Delta$ be an atomic causal theory, $\omega$ a model of the propositional theory obtained by reading the causal rules in $\Delta$ as logical implications, and $I \subseteq \omega$ a set of positive literals true in $\omega$.
    
    An \textbf{$I$-external explanation} is an expression of the form $\lambda \Rrightarrow_{\Delta} l$, where~${l \in I}$ and $\lambda$ is a set of literals that are true in $\omega$ and do not belong to $I$. An \textbf{$I$-internal explanation} is an expression of the form $\lambda \Rrightarrow_{\Delta} l$ that is not $I$-external.
    \label{def:Internal}
\end{definition}

Let $\omega$ be a causally founded world of the Bochman transformation ${\textbf{CS}(\mathcal{P})}$. Definition~\ref{definition - Bochman transformation} ensures that a subset~$I \subseteq \omega$ of atoms true in~$\omega$ can be unfounded with respect to any program extending the underlying logic program~$\textbf{P}$ only if every rule~$\lambda \Rightarrow l$ in~$\Delta$ corresponds to an $I$-internal explanation~$\lambda \Rrightarrow_{\Delta} l$.

\begin{lemma}
    Let $\Delta$, $\omega$ and $I$ be as in Definition \ref{def:Internal}. 
    If every rule $\lambda \Rightarrow l$ in $\Delta$ corresponds to an \mbox{$I$-internal} explanation $\lambda \Rrightarrow_{\Delta} l$, then every other explanation~${\lambda' \Rrightarrow_{\Delta} l'}$ is also~\mbox{$I$-internal} as well.
    \label{Lemma - Internal rules}
\end{lemma}

\begin{proof}
    By Definition \ref{theorem - causal reasoning of literal causal theories},  $\lambda' \Rrightarrow_\Delta l'$ follows from the rules in $\Delta$ through iterated applications of literal cut, literal monotonicity and literal contradiction. So it suffices to show that as long as the rules in the premises to those three inference rules are $I$-internal, then so is their consequence. 
    
    We use the notation of Definition \ref{theorem - causal reasoning of literal causal theories}. 
    
    For literal monotonicity, the statement is clear, since if $\lambda \cup \lambda'$ is a set of literals true in $\omega$ and not in $I$, then so is its subset $\lambda$. 

    Since $\omega$ is a structure, $q$ and $\neg q$ can never both be true in $\omega$.
    Therefore, literals contradiction can never yield an $I$-external explanation. 

    So it only remains to consider literal cut. 
    Assume for contradiction that ${\lambda \cup \lambda' \Rrightarrow_\Delta l'}$ is \mbox{$I$-external} despite both premises $\lambda' \Rrightarrow_{\Delta} l$ and~${\lambda \cup \{ l \} \Rrightarrow_{\Delta} l'}$ being $I$-internal. 
    Then $\lambda \cup \lambda'$ is a set of literals true in $\omega$ and not in $I$, and therefore so are $\lambda$ and~$\lambda'$. Thus, since both premises are $I$-internal, the atom~$l$ must not be in $I$ and also not true in $\omega$. 
    However, this contradicts the fact that $\omega$ is a model of the propositional theory corresponding to the rules of $\Delta$.
    Indeed, note that (literal) cut, monotonicity and contradiction are all true for propositional logic, where the triple arrow is read as logical implication.
    Thus, since $\lambda \Rrightarrow_{\Delta} l$ is obtained from rules in $\Delta$ using those three axioms, $\omega$ is a model of $\lambda \rightarrow l$ and of $\lambda$ and thus of $l$, yielding the desired contradiction.~$\square$
\end{proof}

With Lemma~\ref{Lemma - Internal rules} at hand, we can now prove our first result.

\begin{theorem}[Bochman Transformation]
    The Bochman transformation is a bijection from abductive logic programs to causal systems with default~\mbox{negation}.
    
    An abductive logic program~${\mathcal{P} := (\textbf{P}, \mathfrak{A}, \textbf{IC})}$ has a stable model~$\omega$ if and only if~$\omega$ is a causally founded world with respect to its Bochman transformation~${\textbf{CS}(\mathcal{P}) := (\Delta, \mathcal{E}, \mathcal{O})}$.
    \label{theorem - bochman transformation}
\end{theorem}

\begin{proof}
    By construction the Bochman transformation is a bijection between abductive logic programs and causal systems with default~\mbox{negation}.
    Since the integrity constraints are carried over unchanged by the Bochman transformation as observations, we can assume without loss of generality that~$\mathcal{P}$ is without integrity constraints.
    
    We show that every stable model of  ${\textbf{P} \cup (\omega \cap \mathfrak{A})}$ is a causal founded world of ${\textbf{CS}(\mathcal{P})} =: (\Delta,\mathcal{E},\emptyset)$. 

    So let $\omega$ be such a stable model. 
    We need to show for all literals $l$ that ${\omega \cap \mathcal{E} \Rrightarrow_{\Delta(\textbf{CS})} l}$ if and only if $l \in \omega$. 
    Note that since the rules in $\Delta$ correspond precisely to the clauses of the underlying logic program $\textbf{P}$, the Clark completion of $\textbf{P}$ coincides with the completion of the explanatory closure $\Delta(\textbf{CS})$.
    Since every stable model is also a supported model, the structure $\omega$ is a model of the completion of $\Delta({\textbf{CS}(\mathcal{P})})$. 
    
    For any set $\Lambda$ of literals, we introduce the notation $\mathcal{C}(\Lambda)$ to indicate the set of all literals~$l$ such that $\Lambda \Rrightarrow_{\Delta(\textbf{CS}(\mathcal{P}))} l$.  
    Thus, Theorem \ref{theorem - completion of a causal theory} states that $\mathcal{C} (\omega) = \omega$ and by literal monotonicity $\mathcal{C} (\omega \cap \mathcal{E}) \subseteq \omega$. 
    It remains to show that $\omega \subseteq \mathcal{C} (\omega \cap \mathcal{E})$, or, in other words, that $I :=\omega \setminus \mathcal{C}(\omega \cap \mathcal{E}) = \emptyset$.
    Note first that since all negated literals are in $\mathcal{E}$, $I$ is a set of positive literals.
    We show that if it were non-empty, $I$ were an unfounded set. 
    Assume that~$I$ is not unfounded. 
    Then there would be a $p \in I$ and a clause $p \leftarrow b_1, \dots, b_n$ such that all of $ b_1, \dots, b_n$ are in $\omega \setminus I = \mathcal{C}(\omega \cap \mathcal{E})$. 
    However, this implies that $\omega \cap \mathcal{E} \Rrightarrow_{\Delta({\textbf{CS}(\mathcal{P})})} b_i$  for $i \in 1, \dots, n$ and therefore by $n$-fold iteration of the (literal) cut, that  $\omega \cap \mathcal{E} \Rrightarrow_{\Delta({\textbf{CS}(\mathcal{P})})} p$ and thus that $p \in \mathcal{C}(\omega \cap \mathcal{E})$. 
    This in turn contradicts $p \in I$, and thus concludes the proof that $I = \emptyset$ since stable models contain no non-empty unfounded sets of atoms. 
    Overall, we have shown that $\mathcal{C} (\omega \cap \mathcal{E}) = \omega$ and therefore that $\omega$ is a causally founded world of ${\textbf{CS}(\mathcal{P})}$. 

    We turn to the converse direction, demonstrating that every causally founded world of ${\textbf{CS}(\mathcal{P})}$ is a stable model of $\mathcal{P}$. 
    By Theorem \ref{theorem - completion of a causal theory}, every causally founded world of ${\textbf{CS}(\mathcal{P})}$ is a model of the propositional theory corresponding to the clauses of $\textbf{P}$. 

    So it remains to show that $\omega$ has no unfounded sets of atoms with respect to~${\textbf{P} \cup (\omega \cap \mathfrak{A})}$. 
    Assume it does have such a set, say $I$.

    As $\omega$ is a causally founded world we obtain $\omega \cap \mathcal{E}  \Rrightarrow_{\Delta({\textbf{CS}(\mathcal{P})})} p$ for any~${p \in I}$. Since every abducible atom true in $\omega$ corresponds to a fact of ${\textbf{P} \cup (\omega \cap \mathfrak{A})}$, the unfounded set $I$ must be disjoint from~$ \mathfrak{A}$ and thus from $\mathcal{E}$.
    Therefore,  $\omega \cap \mathcal{E} \Rrightarrow_{\Delta({\textbf{CS}(\mathcal{P})})} p$ is $I$-external. 

    By Lemma \ref{Lemma - Internal rules}, this implies that one of the rules of $\Delta({\textbf{CS}(\mathcal{P})})$ is $I$-external and thus that $I$ is not an unfounded subset of $\omega$. 

    This concludes the proof that $\omega$ is a stable model of ${\textbf{P} \cup (\omega \cap \mathfrak{A})}$ as claimed.~$\square$
\end{proof}

According to Theorem~\ref{theorem - bochman transformation}, Principles~\ref{principle - causal foundation}, \ref{principle - Aquinas}, and \ref{principle - Leibniz}, together with the formalizations in Definition~\ref{definition - causal system}, entail that a causal interpretation of abductive logic programming necessarily yields the stable model semantics of Gelfond and Lifschitz~\cite{StableModelSemantics}. This raises the question of whether every abductive logic program admits such a causal interpretation.

\begin{example}
    Let $e$ denote the event that a farmer is ecological, and $h$ the event that it is hot. Further, let $s$ denote that pests survive the weather, $p$ that there are pests in the field, and $t$ that the farmer applies toxin.

    Assume pests survive if it is hot. If no toxin is applied, pests remain; toxin is applied if pests are present and the farmer is not ecological. These relations define an abductive logic program $\mathcal{P}$ with abducibles $\mathfrak{A} := \{h, e\}$ and rules $\textbf{P}$:
    \begin{align*}
        t &\leftarrow p, \neg e, &
        p &\leftarrow \neg t, s, &
        s &\leftarrow h.
    \end{align*}
    
    The program $\mathcal{P}$ has no stable model with explanation $\epsilon := \{h\}$. Hence, it concludes against the causal direction, namely, that it cannot be hot if the farmer is not ecological. This clearly contradicts everyday intuition.
    \label{example - conclude against causal direction}
\end{example}

We argue that to represent causal knowledge, an abductive logic program must also satisfy Principle~\ref{principle - causal irrelevance}, as explored by Williamson~\cite{Williamson2001} in the context of Bayesian networks. In abductive logic programming, we interpret Principle~\ref{principle - causal irrelevance} as the following semantic constraint:

\begin{formalization}[Principle~\ref{principle - causal irrelevance}]
    Let $\mathcal{P} := (\textbf{P}, \mathfrak{A}, \textbf{IC})$ be a consistent abductive logic program.  
    For a set $S \subseteq \mathfrak{P}$ of propositions, let $\mathfrak{P}^{> S}$ denote the set of all propositions $q \not\in S$ that are descendants in the dependency graph $G := \graph(\textbf{P})$ of some proposition in $S$.  
    Finally, denote by $\textbf{P}^{>S}$ the program consisting of all clauses $C \in \textbf{P}$ with $\head(C) \in \mathfrak{P}^{> S}$.

    Then, $\mathcal{P}$ \emph{satisfies Principle~\ref{principle - causal irrelevance}} if and only if for every set $S \subseteq \mathfrak{P}$ and every $\mathfrak{P} \setminus \mathfrak{P}^{> S}$-structure $\omega$, the program
    $
        \textbf{P}^{>S, \omega} := \textbf{P}^{>S} \cup \omega
    $
    has at least one stable model; that is, it is not possible to falsify $\omega$ with $\textbf{P}^{>S}$.
    \label{formalization - causal irrelevance}
\end{formalization}

\begin{remark}
    Williamson proposes Principle~\ref{principle - causal irrelevance} in the context of maximum entropy as a weakening of the Markov assumption in Bayesian networks~\cite{Causality}. 
    Accordingly, the above formalization could be viewed as a deterministic analogue of the Markov assumption.
\end{remark}

\begin{example}
    The program $\mathcal{P}$ in Example~\ref{example - conclude against causal direction} does not satisfy Principle~\ref{principle - causal irrelevance}.
\end{example}

We argue that every abductive logic program $\mathcal{P}$ satisfying Principle~\ref{principle - causal irrelevance} under the above formalization admits a causal interpretation under the stable model semantics. This raises the question whether $\mathcal{P}$ can be used to predict the effect of external interventions.

According to Pearl~\cite{Causality}, the joint act of observing and intervening leads to counterfactual reasoning, which lies beyond the scope of this contribution. Therefore, we restrict our interest to programs without integrity constraints and argue that they admit a meaningful representation of interventions if Principle~\ref{principle - non-interference} holds. The following result shows that Principle~\ref{principle - causal irrelevance} implies Principle~\ref{principle - non-interference}.

\begin{theorem}
    Let $\mathcal{P} := (\textbf{P}, \mathfrak{A}, \emptyset)$ be an abductive logic program without integrity constraints that satisfies Principle~\ref{principle - causal irrelevance}, and let $\mathbf{i}$ be an assigment on a set of propositions $S \subseteq \mathfrak{P}$. Define $\mathfrak{P}^{< S} := \mathfrak{P}\setminus (\mathfrak{P}^{> S} \cup S)$, ${\textbf{P}^{< S} := \{ C \in \textbf{P} \mid \head(C) \in \mathfrak{P}^{< S} \}}$ and  $\mathcal{P}^{< S} := (\textbf{P}^{< S}, \mathfrak{A}, \emptyset)$. Then, the following are equivalent:
    \begin{enumerate}
        \item $\omega^{< S} = \omega \cap \mathfrak{P}^{< S}$ for some stable model $\omega$ of $\mathcal{P}$, i.e., $\omega^{< S}$ is a reduct of $\omega$.
        \item $\omega^{< S}$ is a stable model of $\mathcal{P}^{< S}$.
        \item $\omega^{< S} = \omega_{\mathbf{i}} \cap \mathfrak{P}^{< S}$ for some stable model $\omega_{\mathbf{i}}$ of $\mathcal{P}_{\mathbf{i}}$, i.e., $\omega^{< S}$ is a reduct of $\omega_{\mathbf{i}}$.
    \end{enumerate}
    \label{theorem - non-interference}
\end{theorem}

\begin{proof}
    The proof rests on the splitting lemma \cite{LifschitzT94}. For any set of propositions $S \subseteq  \mathfrak{P}$ denote by $\mathfrak{P}^{\geq S} := \mathfrak{P}^{>S} \cup S$. Let $\mathcal{P}^S$ be the set of all clauses with heads in $S$, and write $\textbf{P}^{*S}$ for the set of all clauses with heads in $\mathfrak{P}^{*S}$, where $* \in \{<,\geq,>\}$. Finally, set $\mathcal{P}^{\ast S} := (\textbf{P}^{\ast S}, \mathfrak{A}, \emptyset)$.

    Let $\omega$ be a world. By abuse of notation, for an abductive logic program ${\mathcal{Q} := (Q, \mathfrak{B}, \emptyset)}$ we denote the logic program $Q \cup (\omega \cap \mathfrak{B})$ also by $\mathcal{Q}$.
    
    We first show the equivalence of 2 and 3, that stable models of~$\mathcal{P}^{<S}$ are precisely the reducts of stable models of $\mathcal{P}_\mathbf{i}$.
    Note first that  $\mathcal{P}^{< S} = \mathcal{P}_{\mathbf{i}}^{ < S}$. 
    Now consider the splitting $(\mathcal{P}_{\mathbf{i}}^{< S},\mathcal{P}_{\mathbf{i}}^{S},\mathcal{P}_{\mathbf{i}}^{> S})$. 
    This is indeed a splitting, since after intervention the propositions in $S$ have no ancestors at all. 
    Therefore, every reduct of a stable model of $\mathcal{P}_{\mathbf{i}}$ to $\mathfrak{P}^{< S}$ is a stable model of $\mathcal{P}_{\mathbf{i}}^{< S} =\mathcal{P}^{< S}$. 
    For the other direction, let $\omega$ be a stable model of $\mathcal{P}^{< S} = \mathcal{P}_{\mathbf{i}}^{ < S}$. 
    Since $\mathcal{P}_\mathbf{i}^{> S,\omega}$ consists only of facts, it clearly has a stable model, say~$\omega_S$. By assumption, the program $\mathcal{P}^{>S,\omega_S} = \mathcal{P}_\mathbf{i}^{>S,\omega_S}$  has a stable model, which by the splitting lemma is also a stable model of $\mathcal{P}_\mathbf{i}$ and extends $\omega$.  

    Now we turn to the equivalence of 1. and 2. 
    Note first that for any $S$, $(\mathcal{P}^{< S},\mathcal{P}^{\geq S})$ is a splitting of $\mathcal{P}$. Therefore, every reduct of a stable model of $\mathcal{P}$ to $\mathfrak{P}^{< S}$ is a stable model of $\mathcal{P}^{< S}$. 

    So let $\omega$ be such a stable model of $\mathcal{P}^{< S}$. We need to show that $\omega$ can be extended to a stable model of $\mathcal{P}$. 
    Denote $\mathfrak{P}^{< S}$ by $S'$. 
    We note that $S'$ is closed under predecessors since $\mathfrak{P}^{\geq S}$ is clearly closed under successors. 
    We employ the splitting $(\mathcal{P}^{S'},\mathcal{P}\setminus (\mathcal{P}^{S'} \cup \mathcal{P}^{>S'}), \mathcal{P}^{>S'})$.
    Note that the vocabularies used in $\mathcal{P}^{S'}$ and in $\mathcal{P}\setminus (\mathcal{P}^{S'} \cup \mathcal{P}^{>S'})$ are disjoint, since if a head occurring in $\mathcal{P}\setminus (\mathcal{P}^{S'} \cup \mathcal{P}^{>S'})$  would be a successor of a proposition in $S'$, it would lie in $\mathfrak{P}^{>S'}$, and $S'$ is closed under predecessors. 
    Since if the vocabularies of two logic programs are disjoint, the stable models of their union are precisely the unions of their stable models and $\mathcal{P}$ is consistent (therefore $\mathcal{P}\setminus (\mathcal{P}^{S'} \cup \mathcal{P}^{>S'})$ has at least one stable model),  every stable model of $\mathcal{P}^{S'}$ extends to a stable model of $\mathcal{P}\setminus  \mathcal{P}^{>S'}$. 
    The result now follows immediately from Formalization \ref{formalization - causal irrelevance}.~$\square$
\end{proof}

Finally, we obtain that stratified abductive logic programs under the stable model semantics satisfy Principle \ref{principle - causal irrelevance}.

\begin{theorem}
    Every stratified abductive logic program satisfies Principle~\ref{principle - causal irrelevance}.
    \label{theorem - stratified programs satisfy causal irrelevance}
\end{theorem}

\begin{proof}
    If $\mathcal{P} := (\textbf{P}, \mathfrak{A}, \textbf{IC})$ is stratified, then so is $\textbf{P}^{>S,\omega}$ for every $\textbf{S} \subseteq \mathfrak{P}$. 
    Therefore, $\textbf{P}^{>S,\omega}$ has (precisely) one stable model.~$\square$ 
\end{proof} 

\section{Conclusion and Related work}

Let $\mathcal{P} := (\textbf{P}, \mathfrak{A}, \textbf{IC})$ be an abductive logic program. This work proposes a causal interpretation of $\mathcal{P}$ that reads each clause ${h \leftarrow b_1, \ldots, b_n \in \textbf{P}}$ as ``$b_1 \land \cdots \land b_n$ \textit{causes} $h$,'' interprets the abducibles $\mathfrak{A}$ as external premises with explanations beyond the given model, and treats the integrity constraints $\textbf{IC}$ as observations. Theorem~\ref{theorem - bochman transformation}, together with the formalizations in Definition~\ref{definition - causal system}, shows that Principles~\ref{principle - causal foundation}–\ref{principle - Leibniz} entail the stable model semantics~\cite{StableModelSemantics}.

However, Example~\ref{example - conclude against causal direction} illustrates that general programs may violate Principle~\ref{principle - causal irrelevance}. Theorem~\ref{theorem - stratified programs satisfy causal irrelevance} confirms that stratified programs respect this principle, supporting their causal interpretability under the stable model semantics.

In the absence of observations ($\textbf{IC} = \emptyset$), Theorem~\ref{theorem - non-interference} shows that Principle~\ref{principle - causal irrelevance} implies~\ref{principle - non-interference}. Hence, stratified programs without integrity constraints support reliable predictions under interventions.

Several authors have explored the relation between causal logic and logic programming. To our knowledge, the earliest results appeared in the context of causally enriched versions of the situation calculus \cite{Lin95,McCainT97,GiunchigliaL98}. McCain~\cite{McCain97} and Lin and Wang~\cite{LinW99} translate causal constraints of such languages into disjunctive logic programs with classical negation. McCain's transformation~\cite{McCain97} is extended to a broader class of causal theories, including first-order ones, by Ferraris et al.~\cite{FerrarisLLLY12}. This line of work relies on classical negation, thus departing from the standard framework of negation as failure. Its aim is to make causal theories executable, rather than to provide a causal semantics for logic programs.

Conversely, some authors have investigated how logic programs themselves might admit a causal interpretation. This includes work by Giunchiglia et al.~\cite{GiunchigliaLLMT04} and Bochman~\cite{Bochman04}, who, under the stable model semantics, translate a logic programming clause of the form~$c \leftarrow \vec{a}, \neg \vec{b}$ into the causal rule
$
 \bigwedge \neg \vec{b} \Rightarrow \left( \bigwedge \vec{a} \rightarrow c \right).
$
Compared to the Bochman transformation in Definition~\ref{definition - Bochman transformation}, which yields the rule $\vec{a}, \neg \vec{b} \Rightarrow c$, this formulation is more difficult to interpret. In particular, as noted by Eelink et al.~\cite[§2.2.2]{EelinkRW25}, the use of embedded logical implication within a causal reasoning framework is far less transparent than the use of ordinary propositions. This formulation also implies that logic programs correspond only to causal rules of a highly specific syntactic form. From this perspective, logic programs could not express causal dependencies between positive atoms, which would severely limit their causal expressiveness. Moreover, to our knowledge, these works do not address the feasibility of modeling interventions, nor do they consider Principles~\ref{principle - causal foundation}, \ref{principle - causal irrelevance} and~\ref{principle - non-interference}.

The Bochman transformation in Definition~\ref{definition - Bochman transformation} maps facts to rules of the form~${\top \Rightarrow h}$, whose interpretation remains open~\cite[Remark 1.1]{EelinkRW25}. Future work should investigate suitable causal readings of such rules, identify the class of programs satisfying Principle~\ref{principle - causal irrelevance}, and extend the framework to disjunctive programs. 
While this work addresses whether a program can meaningfully represent the effect of all interventions, it remains to be explored when a program contains sufficient information to represent particular interventions, and whether Principle \ref{principle - non-interference}, appropriately formalized, is in fact equivalent to Principle \ref{principle - causal irrelevance}.

%\newpage
 \bibliographystyle{splncs04}
 \bibliography{literature}

\begin{thebibliography}{10}
\providecommand{\url}[1]{\texttt{#1}}
\providecommand{\urlprefix}{URL }
\providecommand{\doi}[1]{https://doi.org/#1}

\bibitem{Anderson}
Anderson, J.F.: Summa Contra Gentiles, 2: Book Two: Creation. University of Notre Dame Press (1956), \url{https://doi.org/10.2307/j.ctvpj74rh}

\bibitem{ArifM22}
Arif, S., MacNeil, M.A.: Applying the structural causal model framework for observational causal inference in ecology. Ecological Monographs  \textbf{93}(1),  e1554 (2023). \doi{https://doi.org/10.1002/ecm.1554}, \url{https://esajournals.onlinelibrary.wiley.com/doi/abs/10.1002/ecm.1554}

\bibitem{Bochman04}
Bochman, A.: A causal logic of logic programming. In: Dubois, D., Welty, C.A., Williams, M. (eds.) Principles of Knowledge Representation and Reasoning: Proceedings of the Ninth International Conference. pp. 427--437. {AAAI} Press (2004), \url{http://www.aaai.org/Library/KR/2004/kr04-045.php}

\bibitem{bochman2005}
Bochman, A.: Explanatory Nonmonotonic Reasoning. World Scientific (2005), \url{https://doi.org/10.1142/5707}

\bibitem{bochman}
Bochman, A.: {A Logical Theory of Causality}. The MIT Press (2021), \url{https://doi.org/10.7551/mitpress/12387.001.0001}

\bibitem{databases}
Chomicki, J., Saake, G. (eds.): Logics for databases and information systems. Kluwer Academic Publishers, USA (1998)

\bibitem{ClarkCompletion}
Clark, K.L.: Negation as failure. In: Logic and Data Bases. pp. 293--322. Springer US, Boston, MA (1978), \url{https://dl.acm.org/doi/book/10.5555/578615}

\bibitem{AbductiveLogicPrograms}
Denecker, M., Kakas, A.C.: Abduction in logic programming. In: Computational Logic: Logic Programming and Beyond. pp. 402--436. Springer (2002), \url{https://doi.org/10.1007/3-540-45628-7\_16}

\bibitem{EelinkRW25}
Eelink, G., Rückschloß, K., Weitkämper, F.: How artificial intelligence leads to knowledge why: An inquiry inspired by {A}ristotle's posterior analytics (2025), \url{https://arxiv.org/abs/2504.02430}

\bibitem{Clark}
Fages, F.: Consistency of {C}lark's completion and existence of stable models. Methods of Logic in Computer Science (1),  51--60 (1994)

\bibitem{FerrarisLLLY12}
Ferraris, P., Lee, J., Lierler, Y., Lifschitz, V., Yang, F.: Representing first-order causal theories by logic programs. Theory and Practice of Logic Programming  \textbf{12}(3),  383--412 (2012). \doi{10.1017/S1471068411000081}

\bibitem{GaoZWFHL24}
Gao, C., Zheng, Y., Wang, W., Feng, F., He, X., Li, Y.: Causal inference in recommender systems: {A} survey and future directions. {ACM} Trans. Inf. Syst.  \textbf{42}(4),  88:1--88:32 (2024). \doi{10.1145/3639048}, \url{https://doi.org/10.1145/3639048}

\bibitem{StableModelSemantics}
Gelfond, M., Lifschitz, V.: The stable model semantics for logic programming. In: Proceedings of International Logic Programming Conference and Symposium. pp. 1070--1080. MIT Press (1988), \url{http://www.cs.utexas.edu/users/ai-lab?gel88}

\bibitem{GiunchigliaLLMT04}
Giunchiglia, E., Lee, J., Lifschitz, V., McCain, N., Turner, H.: Nonmonotonic causal theories. Artificial Intelligence  \textbf{153}(1-2),  49--104 (2004). \doi{10.1016/J.ARTINT.2002.12.001}, \url{https://doi.org/10.1016/j.artint.2002.12.001}

\bibitem{GiunchigliaL98}
Giunchiglia, E., Lifschitz, V.: An action language based on causal explanation: Preliminary report. In: Mostow, J., Rich, C. (eds.) Proceedings of the Fifteenth National Conference on Artificial Intelligence {AAAI} 98. pp. 623--630. {AAAI} Press / The {MIT} Press (1998), \url{http://www.aaai.org/Library/AAAI/1998/aaai98-088.php}

\bibitem{KakasM90}
Kakas, A.C., Mancarella, P.: Generalized stable models: {A} semantics for abduction. In: 9th European Conference on Artificial Intelligence, {ECAI} 1990,. pp. 385--391 (1990)

\bibitem{Leibniz}
Leibniz, G.W.: First truths. In: Philosophical Papers and Letters. pp. 267--271. Springer Netherlands (1989), \url{https://doi.org/10.1007/978-94-010-1426-7\_31}

\bibitem{Lifschitz10}
Lifschitz, V.: Thirteen definitions of a stable model. In: Blass, A., Dershowitz, N., Reisig, W. (eds.) Fields of Logic and Computation. pp. 488--503. Springer (2010), \url{https://doi.org/10.1007/978-3-642-15025-8\_24}

\bibitem{LifschitzT94}
Lifschitz, V., Turner, H.: Splitting a logic program. In: Hentenryck, P.V. (ed.) Logic Programming, Proceedings of the Eleventh International Conference on Logic Programming, Santa Marherita Ligure, Italy, June 13-18, 1994. pp. 23--37. {MIT} Press (1994)

\bibitem{Lin95}
Lin, F.: Embracing causality in specifying the indirect effects of actions. In: Proceedings of the Fourteenth International Joint Conference on Artificial Intelligence, {IJCAI} 95, Montr{\'{e}}al Qu{\'{e}}bec, Canada, August 20-25 1995, 2 Volumes. pp. 1985--1993. Morgan Kaufmann (1995), \url{http://ijcai.org/Proceedings/95-2/Papers/123.pdf}

\bibitem{LinW99}
Lin, F., Wang, K.: From causal theories to logic programs (sometimes). In: Gelfond, M., Leone, N., Pfeifer, G. (eds.) Logic Programming and Nonmonotonic Reasoning, 5th International Conference. pp. 117--131. Springer (1999). \doi{10.1007/3-540-46767-X\_9}, \url{https://doi.org/10.1007/3-540-46767-X\_9}

\bibitem{McCain97}
McCain, N.: Causality in commonsense reasoning about actions. Ph.D. thesis, University of Texas at Austin (1997)

\bibitem{McCainT97}
McCain, N., Turner, H.: Causal theories of action and change. In: Kuipers, B., Webber, B.L. (eds.) Proceedings of the Fourteenth National Conference on Artificial Intelligence {AAAI} 97. pp. 460--465. {AAAI} Press / The {MIT} Press (1997), \url{http://www.aaai.org/Library/AAAI/1997/aaai97-071.php}

\bibitem{Miller2009}
Miller, V., Miller, R.: Ren\'{e} Descartes: Principles of Philosophy. Springer Dordrecht (1982), \url{https://doi.org/10.1007/978-94-009-7888-1}

\bibitem{Causality}
Pearl, J.: {C}ausality. Cambridge University Press (2000), \url{https://doi.org/10.1017/CBO9780511803161}

\bibitem{ruckschloss2021exploiting}
R{\"{u}}ckschlo{\ss}, K., Weitk{\"{a}}mper, F.: Exploiting the full power of pearl's causality in probabilistic logic programming. In: Proceedings of the International Conference on Logic Programming 2022 Workshops. {CEUR} Workshop Proceedings, vol.~3193. CEUR-WS.org (2022), \url{https://ceur-ws.org/Vol-3193/paper1PLP.pdf}

\bibitem{SaccaZ90}
Sacc{\`{a}}, D., Zaniolo, C.: Stable models and non-determinism in logic programs with negation. In: Proceedings of the Ninth {ACM} {SIGACT-SIGMOD-SIGART} Symposium on Principles of Database Systems. pp. 205--217. {ACM} Press (1990), \url{https://doi.org/10.1145/298514.298572}

\bibitem{Williamson2001}
Williamson, J.: Foundations for Bayesian Networks, pp. 75--115. Springer Netherlands (2001), \url{https://doi.org/10.1007/978-94-017-1586-7\_4}

\bibitem{WuPLZSLQLG24}
Wu, X., Peng, S., Li, J., Zhang, J., Sun, Q., Li, W., Qian, Q., Liu, Y., Guo, Y.: Causal inference in the medical domain: a survey. Applied Intelligence  \textbf{54}(6),  4911--4934 (2024). \doi{10.1007/S10489-024-05338-9}, \url{https://doi.org/10.1007/s10489-024-05338-9}

\end{thebibliography}




% LNCS Homepage, \url{http://www.springer.com/lncs}, last accessed 2023/10/25
% \end{thebibliography}
\end{document}